\documentclass[letterpaper,oneside]{amsart}
\usepackage[english]{babel}
\usepackage[backend=biber]{biblatex}
\usepackage{geometry,amssymb}
\usepackage{graphicx}
\usepackage{subcaption}
\usepackage{algorithm}
\usepackage{algpseudocode}
\usepackage{hyperref}
\usepackage{todonotes}
\usepackage{comment}
\usepackage{mathtools}
\usepackage{xcolor}
\usepackage{float}
\newcommand{\TeXmacs}{T\kern-.1667em\lower.5ex\hbox{E}\kern-.125emX\kern-.1em\lower.5ex\hbox{\textsc{m\kern-.05ema\kern-.125emc\kern-.05ems}}}
\newcommand{\assign}{:=}

\newcommand{\nin}{\not\in}
\newcommand{\tmem}[1]{{\em #1\/}}

\newcommand{\tmop}[1]{\ensuremath{\operatorname{#1}}}

\newcommand{\R}{\mathbb{R}}

\newenvironment{tmparmod}[3]{\begin{list}{}{\setlength{\topsep}{0pt}\setlength{\leftmargin}{#1}\setlength{\rightmargin}{#2}\setlength{\parindent}{#3}\setlength{\listparindent}{\parindent}\setlength{\itemindent}{\parindent}\setlength{\parsep}{\parskip}} \item[]}{\end{list}}

\newcounter{tmcounter}

\newtheorem{definition}{Definition}
\newcommand{\nonconverted}[1]{\mbox{}}
\newtheorem{proposition}{Proposition}
{\theoremstyle{remark}\newtheorem{remark}{Remark}}
\newtheorem{theorem}{Theorem}
\begin{document}

\title[Noncommutative Model Selection]{Noncommutative Model Selection for data clustering and dimension reduction using Relative von Neumann Entropy}

\author{Araceli Guzm{\'a}n-Trist{\'a}n}

\author{Antonio Rieser$^\ast$}
\address{Centro de Investigación en Matemáticas, A.C., Calle Jalisco S/N, Colonia Valenciana, Guanajuato C.P. 36023, Guanajuato, México}
\email{araceli.guzman@cimat.mx,antonio.rieser@cimat.mx}
\date{}
\thanks{$^\ast$Corresponding author}
\thanks{Araceli Guzmán-Tristán was supported by the CONAHCYT program "Estancias Posdoctorales por México para la Formación y Consolidación de las y los Investigadores por México". Antonio Rieser was supported by the US National
	Science Foundation under grants No. DMS-1929284 and DMS-1928930, the first while in residence at the Institute for Computational and Experimental Research in Mathematics in Providence, RI, during the "Math + Neuroscience: Strengthening the Interplay Between Theory and Mathematics" program, and the second while in residence
	at the Simons-Laufer Mathematical Sciences Research Institute in the spring of 2024 and in a program supported by the Mathematical Sciences Research Institute in the summer of 2022,
	held in partnership with the the Universidad
	Nacional Aut{\'o}noma de M{\'e}xico. Antonio Rieser was also supported by the
	CONAHCYT Investigadoras y Investigadores por M{\'e}xico Project \#1076 and by
	the grant N62909-19-1-2134 from the US Office of Naval Research Global and the
	Southern Office of Aerospace Research and Development of the US Air Force
	Office of Scientific Research.}
\begin{abstract}
  We propose a pair of completely data-driven algorithms for unsupervised
  classification and dimension reduction, and we empirically study their performance on a number of data sets, both simulated data in three-dimensions and images from the COIL-20 data set. The algorithms take as input a set of points sampled from a
  uniform distribution supported on a metric space, the latter embedded in an ambient metric space, and they output a
  clustering or reduction of dimension of the data. They work by constructing a natural family
  of graphs from the data and selecting the graph which maximizes the relative
  von Neumann entropy of certain normalized heat operators constructed from
  the graphs. Once the appropriate graph is selected, the eigenvectors of the graph Laplacian may be used to reduce the dimension of the data, and clusters in the data may be
  identified with the kernel of the associated graph Laplacian. Notably, these
  algorithms do not require information about the size of a neighborhood or
  the desired number of clusters as input, in contrast to popular algorithms
  such as $k$-means, and even more modern spectral methods such as Laplacian eigenmaps, among others. 
  
  In our computational experiments, our clustering algorithm outperforms $k$-means clustering on data sets with non-trivial geometry and topology, in particular data whose clusters are not concentrated around a specific point, and our dimension reduction algorithm is shown to work well in several simple examples.
\end{abstract}

{\maketitle}

\section{Introduction}

Unsupervised clustering and dimension reduction are two of the most important and difficult problems in
modern data analysis, as well as one of the most ubiquitous, appearing in
nearly every area of data science, including image processing, bioinformatics,
and natural language processing, among others. Most popular unsupervised clustering algorithms, including classical methods such as $k$-means
(\cite{Hastie_et_al_2009}, Section 14.3.6), in addition to many
more recently proposed techniques \cites{Belkin_Niyogi_2003,Shen_et_al_2014,Chazal_Guibas_2013,Abdullahi_etal_2022,MR1981019, doi:10.1126/science.290.5500.2319,cite-key,Coifman_Lafon_2006,Abraham_Robbin_1967,Singer_Wu_2012}, depend on additional data which must be chosen by the 
user in order to run, and which keeps these methods from being completely data-driven. In this paper, we propose new spectral clustering and dimension reduction algorithms in which their free parameters may be chosen by considering the noncommutative information-theoretic aspects of these problems.

Following {\cite{Rieser_FODS_2021}}, we interpret the clustering problem to be the problem of assigning each point in a data set $S$ to the closest connected component of the support $X$ of a probability distribution $\mu$ from
which the data was sampled (perhaps with additional noise). Our algorithm
works by constructing a family of weighted graphs $G_r$ for each $r > 0$, where the
vertices of $G_r$ are the data points and the points are connected if they are
within a distance $r$ of one another, with weights given by the ambient distance between two points. If one
considers the heat semigroup on each graph as a heat flow on the vertices,
then the resulting steady state of the flow on $G_r$ is constant on the
connected components of each $G_r$. We also adopt the model-selection heuristic used in \cite{Rieser_FODS_2021}: that the optimal graph in the family of graphs
should be the one for which the diffusion process generated by the graph Laplacian is relatively local initially,
but relatively global at the steady state. We measure this by calculating the
relative von Neumann entropy between the (normalized) heat operator at $t  =
1$ and the (normalized) heat operator at a second time $t' \gg 1$, the latter of which we
interpret as an estimate of limit of the relative von Neumann entropy between
the two operators as $t' \rightarrow \infty$. Choosing the graph where this
relative von Neumann entropy is maximal then produces a graph which is
balanced between vertices being very connected at short distances but not
particularly connected at long distances. Once the graph is chosen, the number
of connected components of the graph is the dimension of the kernel of its
Laplacian matrix, and each point can then be algorithmically assigned to its
connected component as in \cite{Rieser_FODS_2021} by analyzing the null space of the graph
Laplacian. Alternately, one may use the chosen graph to reduce the dimension of the data set by using the eigenvectors of the graph Laplacian to construct a map from the data set $S$ to $\R^k$ for some small integer $k$ as is typically done in Laplacian Eigenmaps \cite{Belkin_Niyogi_2003} and Diffusion Maps \cite{Coifman_Lafon_2006}. In our experiments, the algorithm improved on the
average relative entropy method in {\cite{Rieser_FODS_2021}} when applied to graphs whose edges were defined using Euclidean balls around each point, and outperformed $k$-means on several examples using simulated data as well as the unprocessed images from the COIL-20 image database.

\subsection{Contributions and Related Work}

A difficult problem inherent to many spectral approaches to dimension reduction and clustering, 
including local linear embedding \cite{doi:10.1126/science.290.5500.2323}, Laplacian eigenmaps 
\cite{Belkin_Niyogi_2003}, Hessian eigenmaps \cite{MR1981019}, isomap \cite{doi:10.1126/science.290.5500.2319}, local 
tangent space alignment \cite{cite-key}, diffusion maps \cite{Coifman_Lafon_2006}, and 
vector diffusion maps \cite{Singer_Wu_2012}, is how to systematically choose the free parameter necessary to run the algorithms. Indeed, most current practitioners simply choose these parameters by hand on an ad hoc basis after some trial and error. To the best of our knowledge, there are two 
earlier works which have proposed methods for resolving this issue. In \cite{Rieser_FODS_2021}, the 
second author of the present article proposed two new methods for automatic 
clustering, each of which consisted of a heuristic for selecting a graph from 
a family of graphs by examining the action of the heat semigroup on a basis 
for the functions from the data set to $\R$, and which then analyzed the 
kernel of the relevant graph Laplacian to identify the connected components of 
the chosen graph. In \cite{Shan_Daubechies_2022-arXiv}, a heuristic based on 
the 
semigroup property of the operators constructed in diffusion maps was 
developed for selecting the free 
parameter in that method. We remark that these algorithms are not 
interchangeable. The semigroup technique from 
\cite{Shan_Daubechies_2022-arXiv} is not expected to work in the setting of the current 
article (or that in \cite{Rieser_FODS_2021}), because the free parameter $r>0$ 
in \cite{Rieser_FODS_2021} and the present article is no longer a priori coupled to the semigroup parameter. Conversely, we do not expect that either the Average Local Volume Method or the Average Relative Entropy Method from \cite{Rieser_FODS_2021} can be make to work with diffusion maps, since they both depend on the fact the the graphs in question are disconnected for sufficiently small radii, and all the graphs in diffusion maps are connected (and, indeed, they are even complete graphs on the vertices).

In this article, we make two main contributions to the literature on spectral 
clustering and dimension reduction. First, we introduce a new spectral method 
which uses the ambient distances between points (up to some distance $r$) for 
the weights in our graphs, instead of the similarity kernels which are more 
typically used, and we demonstrate empirically that the spectral properties of 
the Laplacians and heat semigroups of such graphs may also be used for 
clustering and dimension reduction. This simplifies the construction of the 
graphs relative to the standard constructions, and it also clarifies that 
graph approximations of a metric space can be taken to be geometric models of 
the underlying space, and, furthermore, that the diffusion processes analyzed 
in the subsequent clustering and dimension reduction techniques may be built 
on top of the geometry encoded in the edge weights of the graphs, instead of 
guessed a priori. Second, and most importantly, we propose a method for 
selecting a geometric graph model of a metric space from a one-parameter 
family of such models. We do so by maximizing the relative von Neumann entropy 
between an operator in the heat semigroup for a graphs $\{G_r\}_{r>0}$ at some 
finite time (in this case $t=1$) and an approximation of the operator 
representing the steady state of the diffusion process. This builds and 
improves on the methods introduced in \cite{Rieser_FODS_2021}, and it 
introduces genuinely noncommutative tools into the model selection process for 
data clustering and dimension reduction. We give a number of experimental 
demonstrations that the choices of edge weights and model selection techniques 
introduced here produce good results for the data clustering and dimension 
reduction problems, using both simulated and real data sets. 

\section{Graph models and model selection}
    \label{sec:Graph models}
\subsection{Graphs, Laplacians, and Heat Semigroups}\label{subsec:Graphs}

As in {\cite{Rieser_FODS_2021}}, we assume that our data set $S$ has been
sampled from a disconnected metric measure space $(X, d_X, \mu)$, possibly with noise, and we
wish to assign each point $x \in S$ to the closest connected component $X_i
\subset X$. The number of connected components of $X$ is a topological
invariant, and so one of our aims will be to incorporate topological
techniques into the solution to this problem, following the general ideas in
{\cite{Rieser_FODS_2021}}.

We begin with some basic preliminaries on graph Laplacians and the semigroups they generate. For any finite, weighted, undirected graph $G = (V, E, w)$, where $G$ contains no loops,
i.e. $(x, x) \nin E$ for all $x \in V$, and $w : E \to (0, \infty)$ is a
function which assigns a positive weight to every edge, we define the {\tmem{Laplacian
matrix of the graph $G$}}, $L_G$, by
\begin{equation}
  \label{eq:L} (L_G)_{(i, j)} = \left\{\begin{array}{ll}
    - w (x_i, x_j) & \text{if } (x_i, x_j) \in E,\\
    0 & \text{if } (x_i, x_j) \nin E \text{ and } i \neq j\\
    \sum_{(x_i, x_k) \in E} w (x_i, x_k) & \text{if } i = j.
  \end{array}\right.
\end{equation}
We also define the corresponding heat operators $e^{- tL_G}$ for $t \in [0,
\infty)$. Note that $e^{- tL_G} e^{- sL_G} = e^{- (t + s) L_G}$, so $\{e^{-
tL_G} \}_{t \in [0, \infty)}$ forms a semigroup under matrix multiplication,
which we call the {\tmem{heat semigroup of }}$G$.

To find the connected components of a graph $G$, we will appeal to the
following well-known facts (see, for instance, Lemma 1.7(iv) in
{\cite{Chung_1997}}):

\begin{theorem}
  The number of connected components of a graph $G$ is equal to the dimension
  of the kernel of $L_G$.
\end{theorem}

\begin{theorem}
  Each eigenfunction $f \in \ker L_G$ is constant on each connected component
  of $G$.
\end{theorem}

Note that the operators $L_G$ and $e^{-tL_G}$ act on the space of functions $\{f:V_G \to \R\}$. The above
theorems imply that the clustering problem can be solved if we can find a basis for $\ker L_G$ such that
each basis function is non-zero and constant on one of the clusters and zero otherwise.

For the dimension reduction problem, we wish to embed the data set $S$ in to $\R^k$ in a way which preserves the local structure of $S$ as much as possible. The spectral methods for doing so all are loosely based on the following classical theorem by Bérard, Besson, and Gallot \cite{Berard_Besson_Gallot_1994} on embeddings of a Riemannian manifold into $\ell^2$. We begin with a preliminary definition.

\begin{definition}
    Let $M$ be an $n$-dimensional closed manifold and let $a\coloneqq \{\phi_j\}_{j\geq 0}$ be an orthonormal basis of the Laplacian of $M$. Define the family of maps $\psi^a_t: M \to \ell^2$, $t>0$ by
    \[ \psi^a_t(x) \coloneqq \sqrt{2}(4\pi)^{n/4}t^{(n+2)/4}\left
    \{e^{-\lambda_jt/2}\phi^a_j(x)\right\}_{j\geq 1} \]
\end{definition}
\begin{theorem}
    Let $(M,g)$ be a closed Riemannian manifold and let $a = \{\phi^a_j\}_{j\geq0}$ be an orthonormal basis of its Laplacian. Let $g_{E}$ denote the usual Euclidean scalar product on $\ell^2$. then
    \begin{enumerate}
        \item For all positive $t$, the map $\psi^a_t$ is an embedding of $M$ into $\ell^2$.
        \item the pull-back metric $(\psi^a_t)^{\ast}g_{can}$ is asymptotic to the metric $g$ of $M$ when $t \to 0_+$.
    \end{enumerate}\end{theorem}
For the purposes of dimension reduction, this theorem implies that, for some fixed but sufficiently small $t$, the map $S\to \R^k$ given by
\[ 
x \mapsto \{\phi^a_t(x)\}_{1 \leq j\leq k}
\]
may be seen as an approximation of the map $\psi^a_t$, up to a multiplicative constant which depends on the intrinsic dimension of the manifold, and it therefore approximately preserves the local geometry encoded in the Riemannian metric (again, up to some constant coefficients).

\subsection{The Clustering and Dimension Reduction Problems} 

Our approach to both the data
clustering and dimension reduction problems center around choosing a weighted
graph which best estimates the intrinsic geometry of the data from a family of possibile graphs. 

We now define the family of weighted graphs we will consider. Let $S \subset (X,d_X)$ be a finite subset of a metric space $(X,d_X)$ which itself has been embedded in another metric space $(Z,d_Z)$. We further suppose that, for every $x \in X$, 
\begin{equation*}
    \lim_{X \ni x' \to x} \frac{d_Z(x,x')}{d_X(x,x')} =1.
\end{equation*}
This guarantees that $d_Z(x,x')$ and $d_X(x,x')$ are close for $d_X(x,x')$ (or $d_Z(x,x')$) sufficiently small.
For each real number $r > 0$, let $G_r = (V_r, E_r,w)$ be a weighted graph, where $V_r$, $E_r$, and $w:E_r \to \R$ are defined by
\begin{align*}
  V_r & =  S,\\
  E_r & =  \{ (x, x') \in V_r \times V_r \mid d (x, x') \leq r \}\\
  w(x,x') & =  d_{Z}(x,x')
\end{align*}
That is, the vertices of $G_r$ are the points in the data set, and two
vertices are connected by an edge if they lie in a ball of radius $r$ centered
at one of them, and the weight of each edge is the ambient distance in $Z$ between the vertices. Note that these weights are quite different from those used in in Laplacian Eigenmaps \cite{Belkin_Niyogi_2003} or Diffusion Maps \cite{Coifman_Lafon_2006}, where the weights use a heat kernel $K_r:\R^n \times \R^n \to \R$ defined on the ambient space $\R^n$, which is meant to provide a local estimate of a diffusion kernel on $X$ evaluated at the vertices of each edge. Also the graphs $G_r$ are not fully connected in general, unlike in Laplacian eigenmaps and diffusion maps. No edge is added between any pair of vertices $x,x'$ with $d_{Z}(x,x')> r$. In fact, the graphs $G_r$ will be completely disconnected for sufficiently small $r>0$.

Once we have the graphs $G_r$, to solve the clustering
problem, we must choose a scale ${\hat{r}} > 0$ so that the connected
components of the graph $G_{\hat{r}}$ at this scale best approximate the connected components of $X$. 
For the dimension reduction problem, on the other hand, we would like to choose a scale $\hat{r}>0$ and an embedding $\Phi:S \to \R^k$, $k<n$, of the form
\[
\Phi(x) \coloneqq (\hat{\phi}_0(x),\hat{\phi}_1(x),\dots,\hat{\phi}_k(x)),
\]
where the $\hat{\phi}_i$ are the eigenfunctions of $L_{G_{\hat{r}}}$, so that the local geometry of $(X,d_X)$ restricted to $S$ is as well-preserved as possible in the image of $S \subset \R^k$.

\subsection{Relative von Neumann Entropy}

Quantum information theory has provided many new tools and insights with which
to study linear operators and operator algebras, motivated by the need to
provide a solid theoretical foundation for quantum computation and quantum
communication. While the computational setting of the present work is
unequivocally classical, we are nonetheless confronted with families of
noncommutative Hermitian operators, a setting in which many of the
constructions of quantum information theory are quite natural. In the algorithms and
experiments which follow, we will see that these quantum constructions not
only apply in this context, but they also reveal important information about
the collection of graphs $\{ G_r \}_{r > 0}$ and they are essential to our model selection algorithms. In this section, we collect the basic definitions and results from quantum information theory that we will require.

\begin{definition}
  Let $H$ be a Hilbert space. A {\tmem{positive operator}} $A$ on $H$ is
  defined to be an operator such that for any vector $v \in H$, the inner
  product $\langle v, Av \rangle_H$ is a real, non-negative number. If, in
  addition, $\langle v, Av \rangle_H > 0$ for all $v \neq 0$, then we say that
  the operator $A$ is {\tmem{positive definite}}.
\end{definition}

\begin{remark}
  We recall that any positive operator has non-negative eigenvalues, and any
  positive definite operator has strictly positive eigenvalues. See \cite{Conway_1990}
  for this and other properties of positive operators on Hilbert spaces.
\end{remark}

Following physics terminology, we define a density operator as follows.

\begin{definition}
  A positive operator $\rho$ is called a {\tmem{density operator}} iff
  $\text{Tr} (\rho) = 1$.
\end{definition}

We now define the relative von Neumann entropy, also known as the relative quantum entropy.

\begin{definition}
  Let the {\emph{support}} of an operator $\rho$ on a Hilbert space $H$ be the
  set
  \[ \text{supp } \rho \assign \{v \in H \mid \rho (v) \neq 0\} . \]
  Suppose that $\rho$ is a density operator, and let $\sigma$ be a positive operator. We define the {\emph{relative von Neumann entropy}} or 
  \emph{relative quantum entropy}
  $H (\rho | | \sigma)$ of $\rho$ and $\sigma$ by
  \[ H (\rho | | \sigma) \assign \left\{\begin{array}{ll}
       \text{Tr} (\rho \log (\rho) - \rho \log (\sigma)) & \text{if supp }
       \rho \subseteq \text{supp } \sigma\\
       + \infty & \text{Otherwise.}
     \end{array}\right. \]
\end{definition}

Although the relative von Neumann entropy is neither symmetric nor satisfies the
triangle inequality, it nonetheless provides a useful way to compare positive
operators with trace between $0$ and $1$. In particular, we have

\begin{proposition}{\cite{Wilde_2013}, Theorem 11.8.2}
If $\rho$ is a density operator and $0 \leq
\text{Tr} (\sigma) \leq 1$, then $H (\rho | | \sigma)$ is non-negative. In addition, under
these conditions, $H
(\rho | | \sigma) = 0$ iff $\sigma = \rho$. 
\end{proposition}

In some special cases, the relative von Neumann entropy reduces to the
relative Shannon entropy of the eigenvalues of the operators, viewed as
distributions on a finite space. We will use following case in our algorithms.

\begin{proposition}
  Let $\rho$ be a density operator on a finite dimensional vector space $V$.
  If $\sigma$ is a positive semi-definite operator such that supp $\rho
  \subseteq \text{supp } \sigma$ and $\rho$ and $\sigma$ are simultaneously
  diagonalizable, then
  \[ H (\rho || \sigma) = \sum_i \lambda^{\rho}_i \log \lambda^{\rho}_i -
     \lambda^{\rho}_i \log \lambda^{\sigma}_i, \]
  where the $\lambda^{\rho}_i$ are the eigenvalues of $\rho$ and the
  $\lambda^{\sigma}_i$ are the eigenvalues of $\sigma$.
\end{proposition}

\begin{proof}
   If the matrices $\rho$ and $\sigma$ are simultaneously diagonalizable, then
  the expression above is the sum of the eigenvalues of $\rho \log \rho - \rho
  \log \sigma$, which is equal to the trace.
\end{proof}

\subsection{Selecting the Model}

We now describe our model selection procedure for both the clustering and dimension reduction problems, which may be seen as a genuinely noncommutative version of the Average Maximum Relative Entropy Method in {\cite{Rieser_FODS_2021}}. Let $S$ be a collection of $n$ points in a metric space
$(X, d)$, and for each $0< r < \text{diam}(S)$, let $G_r$ be the graph
defined in Section \ref{subsec:Graphs}. Let $L_r$
denote the graph Laplacian of $G_r$ and
let $\{e^{- tL_r} \}_{t = 0}^{\infty}$ be the resulting heat semigroup. The heuristic behind our model
selection algorithm is that the operators $e^{- t L_r}$ for low values of $t$ reflect the local combinatorial and geometric structure of the graph $G_r$. That is, for low values of $t$, given a function $f:V\to V$ which takes the value $1$ at a vertex $v$ and is $0$ elsewhere, the function $e^{-t L_r}f$ is supported close to $v$. On the other hand, when $t \to \infty$
the heat semigroup converges to a steady state, erasing the local geometry,
but giving the connected components of $G_r$. We wish to choose an $r > 0$ such that the local geometry of $G_r$ reflected in $e^{-
L_r}$ (with $t = 1$) is as different as possible from the steady state
$\lim_{t \to \infty} e^{- tL_r}$. We might like to measure this difference by the limit
of the relative entropies as $t \rightarrow \infty$, however, since the trace
of the heat operators $\tmop{Tr} (e^{- t L_r})$ may be larger than $1$ a
priori, we first normalize the operators before calculating the relative entropy. Our experiments confirm that this is a useful metric.
Our selected scale $\hat{r}$ will therefore be
\begin{equation}
\label{eq:Scale}
\hat{r} \assign \text{argmax} \left( \lim_{t \to \infty} H \left(
   \frac{1}{\tmop{Tr} (e^{-L_r})} e^{- L_r} \right| 
   \left| \frac{1}{\tmop{Tr} (e^{-t L_r})} e^{- tL_r}
   \right) \right),
\end{equation}
In practice, it is sufficient to use $t \gg 1$, large enough so that the eigenvalues of $e^{-tL_r}$
are close to either $0$ or $1$. Once we have computed $\hat{r}$, we construct a map
$\Psi : V \to \mathbb{R}^{\dim \ker L_{\hat{r}}}$ as in
{\cite{Rieser_FODS_2021}} (see also Algorithm \ref{alg:Gaussian Elimination}
below) which sends the vertices in the $i$-th connected component
of $G_{\hat{r}}$ to the standard basis vector $e_i \in \mathbb{R}^{\dim \ker
L_{\hat{r}}}$. To identify the clusters, we then take the inverse image of each of the
$e_i \in \mathbb{R}^{\dim \ker L_{\hat{r}}}$.

\section{The Algorithms}
\label{sec:Algorithms}
We present three algorithms below: the clustering algorithm, a modified Gaussian elimination algorithm used in the clustering algorithm to identify the clusters from the kernel of a graph Laplacian, and the dimension reduction algorithm. In both the clustering and dimension reduction algorithms, the input is a finite collection of points $S$ in a metric spaces, and for every $r < \text{Diam}(S)$, we construct the graph $G_r$ using the Euclidean distance between points, and we construct $L_r$, $e^{-L_r}$, and $e^{-t^*L_r}$ as in Section \ref{subsec:Graphs}, where $t^*\gg 0$ is sufficiently large that all the eigenvalues of $e^{-t^*L_r}$ are either close to $0$ or close to $1$. The target scale in both cases is chosen according to Equation (\ref{eq:Scale}). In the clustering algorithm, we calculate the kernel of the Laplacian at the target scale, and then use the modified Gaussian elimination algorithm (Algorithm \ref{alg:Gaussian Elimination}) to identify the clusters. In the dimension reduction algorithm (Algorithm \ref{alg:DimRed} below), once we have identified the target scale, we use the first $k$ eigenvectors, $k<n$ of the Laplacian $L_{\hat{r}}$ to define a map 
$\Psi:S\to \R^k$ by $\Psi(x) = (\phi_1(x),\dots,\phi_k(x))$. The dimension $k$ of the dimension reduction is chosen by the user.

\begin{algorithm}[h]
\caption{Clustering Algorithm}\label{alg2}
\begin{algorithmic}[1]
  \For{$r<Diam(S)$} 
  \State{Compute $G_r, L_r, \exp (- L_r)$ and estimate
  l{\'i}m$_{t \to \infty} \exp (- tL_r)$ by $\exp (- t^{\ast} L_r)$ for some
  $t^{\ast}$ large.}
  
  \State{Compute the relative von Neumann entropy $S_r  (\rho || \sigma)$ where $\rho =
  \frac{1}{tr (\exp (- L_r))} \exp (- L_r)$ and $\sigma = \frac{1}{tr (\exp (-
  1000 L_r))} \exp (- 1000 L_r)$ }
 \EndFor
  \State{Define $\hat{r} \assign \text{argmax} S_r$.}
  \State{Compute a basis for thee kernel of $L_{\hat{r}}$, i.e. $\phi_i$ for $i \in 1, \ldots,
  k$.}
  
  \State{Using Algorithm \ref{alg:Gaussian Elimination} and the $\phi_i$, create the map $\Psi : z_m \to \Psi (z_m) = (\psi_1
  (z_m), \psi_2 (z_m), \ldots, \psi_k (z_m)) \in \mathbb{R}^k$}.
  
  \State{Compute the distances $d_i (z_m) \coloneqq || \Psi (z_m) - e_i ||$ for each
  point $z_m$ in the sample.}
  
  \State{Assign the vertex $m$ to the $i$-th cluster if $d_i (z_m) < d_j (z_m)$
  for all $j \neq i$.}
\end{algorithmic}
\end{algorithm}

The modified Gaussian elimination algorithm (Algorithm \ref{alg:Gaussian Elimination} below) takes a matrix whose rows span the kernel of a graph Laplacian - and so each row is constant on connected components of the graph - and outputs a matrix whose entries are either (very close to) 1 or (very close to) 0, and where each row is supported (up to a small error) on exactly one connected component of the graph. The clusters are then identified as the support of each row. This algorithm was also used in the clustering methods in \cite{Rieser_FODS_2021}

\begin{algorithm}[h]    
\caption{Modified Gaussian elimination $\Psi${\smallskip}}\label{alg:Gaussian Elimination}
\begin{algorithmic}[1]
      \For{$i = 1$ to $k$}\State{Reorder columns $i$ through $n$ of
      $\Psi$ so that $| \Psi_{(i, i)} |$ is the maximum of $| \Psi_{(i, j)} |$
      in row $i$.}
      
      \State{Divide row $i$ by $\Psi_{(i, i)}$}
      
      \State{Using elementary row operations, make $\Psi_{(k, i)} = 0$
      for $k \neq i$.}
      \EndFor
      \State{Redefine $\psi_i \assign \Psi_{i, \ast}$, and (abusing
      notation) using the new $\psi_i$, redefine the map $\Psi (z_m) \assign
      (\psi_1 (z_m), \psi_2 (z_m), \ldots, \psi_k (z_m))$.}
\end{algorithmic}
\end{algorithm}

\begin{algorithm}[h]
\caption{Dimension Reduction Algorithm}\label{alg:DimRed}
\begin{algorithmic}[1]
  \For{$r<Diam(S)$} 
  \State{Compute $G_r, L_r, \exp (- L_r)$ and estimate
  l{\'i}m$_{t \to \infty} \exp (- tL_r)$ by $\exp (- t^{\ast} L_r)$ for some
  $t^{\ast}$ large.}
  
  \State{Compute the relative von Neumann entropy $S_r  (\rho || \sigma)$ where $\rho =
  \frac{1}{tr (\exp (- L_r))} \exp (- L_r)$ and $\sigma = \frac{1}{tr (\exp (-
  1000 L_r))} \exp (- 1000 L_r)$ }
 \EndFor
  \State{Define $\hat{r} \assign \text{argmax} S_r$.}
  \State{Compute the eigenvalues of $L_{\hat{r}}$ and sort them in ascending order $\lambda_0=0<\lambda_1<\ldots<\lambda_n<1$}.
  \State{Discard the $0$ eigenvalue and take the corresponding eigenvectors $\Psi_1,\Psi_2,\ldots,\Psi_n$.}
  \State{Let $k < n$ be the target dimension, then the embedding map from $X$ to $\mathbb{R}^k$ is $\Psi(x)=\begin{pmatrix}
    \phi_1(x) \\
   \phi_2(x) \\
  \vdots \\
  \phi_k   (x)
  \end{pmatrix}$.}
\end{algorithmic}
\end{algorithm}

\section{Experimental Results}

\subsection{Clustering} We now present the results of the numerical experiments we ran to test the data clustering and dimension reduction algorithms, as well as comparisons to $k$-means clustering. We tested the data clustering algorithm on both synthetic data and the unprocessed COIL-20 image database \cite{COIL-20}, the latter of which consists of 72 images of each of five objects, where each object is rotated five degrees around a vertical axis between one image and the next. The dimension reduction algorithm was tested on a number of shapes in $\R^3$ and then reduced to shapes in $\R^2$ in order to facilitate the visualization of the results.
\subsection{Relative Entropy Clustering Results}
Tables \ref{table:500pts} and \ref{table:1000pts} summarize the results (the number $\beta_{0}$ of clusters) produced by the clustering algorithm on data sets of $500$ and $1000$ points sampled from three interlinked circles embedded in $\mathbb{R}^{3}$ with a small amount of Gaussian noise (standard deviation SD). Images of the samples, colored according to the results of the clustering algorithm, are shown in Figures  \ref{fig:Ruido01_cor}-\ref{fig:Ruido05_inc}.   The horizontal circle has radius $1$ and center $(0,0,0)$, and the other have radii $0.5$ and $0.4$ and centers $(0,-1,0)$ and $(0,1,0)$, respectively. For the clustering algorithm, the graphs $G_r$ were obtained the Euclidean distance between points as the edge weights and the calculation of the relative entropy used $t^{\ast} = 1000$. We ran five different trials, where the Gaussian noise had standard deviations of $0.01$, $0.02$, $0.03$, $0.04$, and $0.05$. For each standard deviation value, we repeated the experiment $150$ times. In each trial, the relative von Neumann entropy was computed for $200$ values of $r$.

\begin{table}[h]
\begin{center}

\caption{Relative von Neumann entropy Method, 500 sample points} 	\label{table:500pts}
\begin{tabular}{|c|c|c|c|c|}
\hline
\bf{$SD | \beta_{0}$} & \bf{1} & \bf{2} & \bf{3} &  \bf{$>$4} \\ \hline

\bf{0.01} & 0 & 6.667 & 93.334 & 0 \\ \hline
\bf{0.02} & 0 & 64.667 & 35.334 & 0 \\ \hline
\bf{0.03} & 2.667 & 92 & 5.334 & 0 \\ \hline
\bf{0.04} & 31.334 & 68.667 & 0 & 0 \\ \hline
\bf{0.05} & 77.334 & 22.667 & 0 & 0 \\ \hline
 
\end{tabular}
\end{center}
\caption{Each row in the table contains the results of the trial for points samples at the corresponding noise level (SD in the table). The number in each cell is the percent of the experiments for that noise level whose output ($\beta_0$) was the value at the top of the column.}
\end{table}

\begin{table}[h]

\begin{center}

\caption{Relative von Neumann entropy Method, 1000 sample points}	\label{table:1000pts} 
\begin{tabular}{|c|c|c|c|c|}
\hline
\bf{$SD | \beta_{0}$} & \bf{1} & \bf{2} & \bf{3} & \bf{$>$4} \\ \hline

\bf{0.01} & 0 & 0 & 100 & 0   \\ \hline
\bf{0.02} & 0 & 0 & 100 & 0   \\ \hline
\bf{0.03} & 0 & 1.334 & 98.667 & 0  \\ \hline
\bf{0.04} & 0 & 51.334 & 48.667 & 0  \\ \hline
\bf{0.05} & 0 & 96.667 & 3.334 & 0  \\ \hline
 
\end{tabular}
\end{center}
\caption{Each row in the table contains the results of the trial for points samples at the corresponding noise level (SD in the table). The number in each cell is the percent of the experiments for that noise level whose output ($\beta_0$) was the value at the top of the column.}
\end{table}

Comparing these results with those obtained in the model selection by average relative entropy (the ARE method) in \cite{Rieser_FODS_2021}, we see that this method represents a significant improvement. For example, in the case of 500 points and standard deviation of noise equal to $0.01$, the ARE method in \cite{Rieser_FODS_2021} returns the correct number of clusters for 64\% of the trials, compared with the 93.334\% in this method. In the case of 1000 points and standard deviation of noise equal to $0.01$, the ARE method in \cite{Rieser_FODS_2021} returns the correct number of clusters for 62.67\% of the trials, compared with the 100\% in this method.


In Figures \ref{fig:Ruido01_cor}-\ref{fig:Ruido05_inc}, we give several clustering results of the relative von Neumann entropy algorithm with different amounts of noise (as the ones in the table) and its respective entropy vs scale graph for a $500$ sample points. The first of these is an example where the algorithm correctly classifies the clusters, and the second is where the algorithm fails. These were typical results for these trials, i.e. when the algorithm reported the correct number of clusters, the resulting clustering was also correct, and when it reported an incorrect number of clusters, the algorithm typically combined two or more clusters into one. We see from the figures that, even when our method reports an incorrect number of clusters, the reported clusters are well-separated from the others. The resulting classification will still likely useful to an end user in such cases, and typically reflects the existence of a genuine gap between the clusters. Higher amounts of noise reduce the performance of the algorithm.

\begin{figure}[h!]

\centering
\begin{subfigure}[]{0.4\linewidth}
\includegraphics[width=\linewidth]{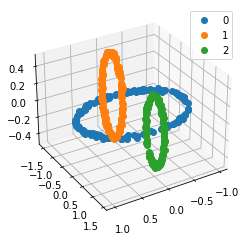}
\end{subfigure}
\begin{subfigure}[]{0.45\linewidth}
\includegraphics[width=\linewidth]{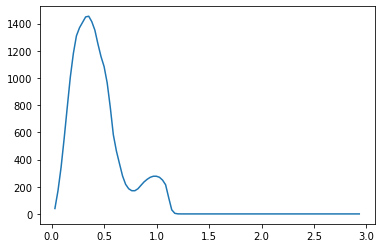}
\end{subfigure}
\caption{A typical example showing the correct classification of clusters. Noise SD=0.01. Left: Output of the algorithm. Right: Graph of relative von Neumann entropy (y-axis) vs. scale (x-axis).}	\label{fig:Ruido01_cor}
\end{figure}

\begin{figure}[h!]
\centering
\begin{subfigure}[]{0.4\linewidth}
\includegraphics[width=\linewidth]{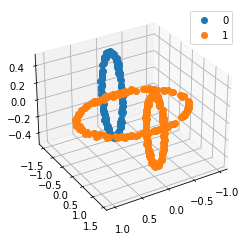}
\end{subfigure}
\begin{subfigure}[]{0.45\linewidth}
\includegraphics[width=\linewidth]{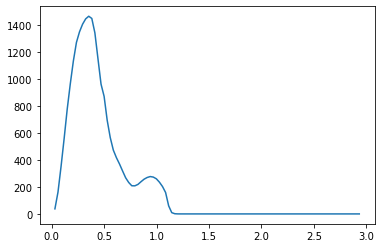}
\end{subfigure}
\caption{Incorrect classification of clusters. Noise SD=0.01. Left: Output of the algorithm. Right: Graph of Entropy (y-axis) vs. Scale (x-axis).}
\end{figure}

\begin{figure}[h!]
\centering
\begin{subfigure}[]{0.4\linewidth}
\includegraphics[width=\linewidth]{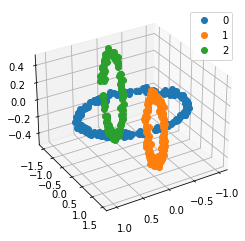}
\end{subfigure}
\begin{subfigure}[]{0.45\linewidth}
\includegraphics[width=\linewidth]{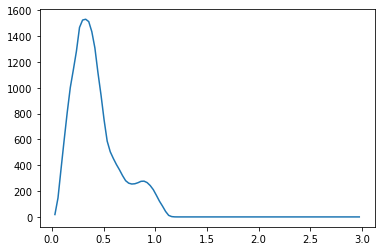}
\end{subfigure}
\caption{Correct classification of clusters. Noise SD=0.02. Left: Output of the algorithm. Right: Graph of Entropy (y-axis) vs. Scale.}
\end{figure}

\begin{figure}[h!]
\centering
\begin{subfigure}[]{0.4\linewidth}
\includegraphics[width=\linewidth]{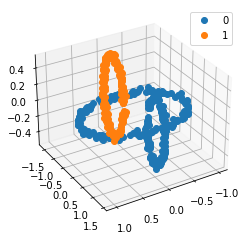}
\end{subfigure}
\begin{subfigure}[]{0.45\linewidth}
\includegraphics[width=\linewidth]{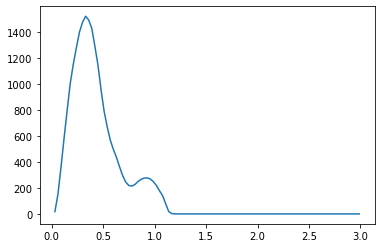}
\end{subfigure}
\caption{Incorrect classification of clusters. Noise SD=0.02. Left: Output of the algorithm. Right: Graph of Entropy (y-axis) vs. Scale (x-axis).}
\end{figure}

\begin{figure}[h!]
\centering
\begin{subfigure}[]{0.4\linewidth}
\includegraphics[width=\linewidth]{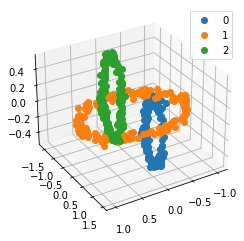}
\end{subfigure}
\begin{subfigure}[]{0.45\linewidth}
\includegraphics[width=\linewidth]{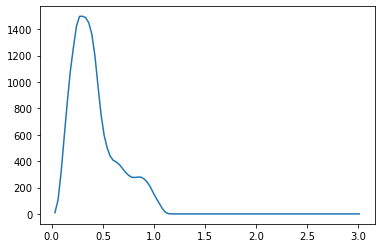}
\end{subfigure}
\caption{Correct classification of clusters. Noise SD=0.03. Left: Output of the algorithm. Right: Graph of Entropy (y-axis) vs. Scale (x-axis).}
\end{figure}

\begin{figure}[h!]
\centering
\begin{subfigure}[]{0.4\linewidth}
\includegraphics[width=\linewidth]{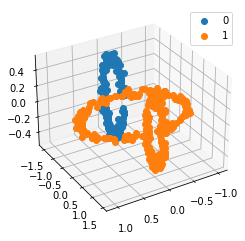}
\end{subfigure}
\begin{subfigure}[]{0.45\linewidth}
\includegraphics[width=\linewidth]{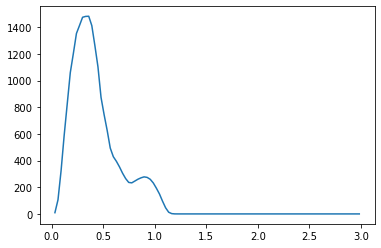}
\end{subfigure}
\caption{Incorrect classification of clusters. Noise SD=0.03. Left: Output of the algorithm. Right: Graph of Entropy (y-axis) vs. Scale (x-axis).}
\end{figure}


\begin{figure}[h!]
\centering
\begin{subfigure}[]{0.4\linewidth}
\includegraphics[width=\linewidth]{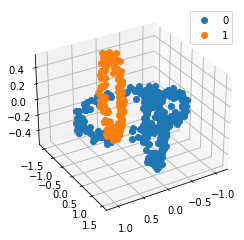}
\end{subfigure}
\begin{subfigure}[]{0.45\linewidth}
\includegraphics[width=\linewidth]{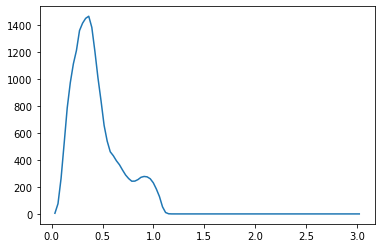}
\end{subfigure}
\caption{Incorrect classification of clusters. Noise SD=0.04. Left: Output of the algorithm. Right: Graph of Entropy (y-axis) vs. Scale (x-axis).}
\end{figure}


\begin{figure}[h!]
\centering
\begin{subfigure}[]{0.4\linewidth}
\includegraphics[width=\linewidth]{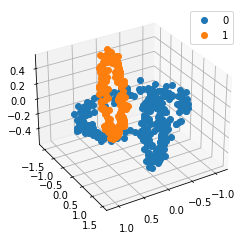}
\end{subfigure}
\begin{subfigure}[]{0.45\linewidth}
\includegraphics[width=\linewidth]{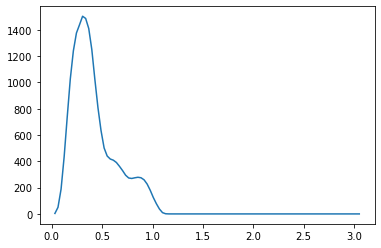}
\end{subfigure}
\caption{Incorrect classification of clusters. Noise SD=0.05. Left: Output of the algorithm. Right: Graph of Entropy (y-axis) vs. Scale (x-axis).}\label{fig:Ruido05_inc}
\end{figure}

\subsection{Comparison with other methods}
We also compared our clustering algorithm with the \emph{$k$-means} algorithm. For a 1000 points sample in the three interlinked circles we ran the three algorithms, in different amount of Gaussian noise. In addition to the  sample of points, for the $k$-means algorithm we took as input the correct number of clusters $k=3$. The results of these experiments are shown in the Figures \ref{fig:comparison_0}-\ref{fig:comparison_5}. We see from the figures that the $k$-means algorithms incorrectly identified the clusters. Nonetheless, we consider this unsurprising, as these two algorithms are known to perform poorly on data in which the clusters have interesting geometry and which are not concentrated at a point.
\begin{figure}[h!]
\centering
\begin{subfigure}[]{0.25\linewidth}
\includegraphics[width=\linewidth]{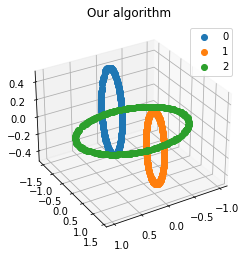}
\end{subfigure}
\begin{subfigure}[]{0.39\linewidth}
\includegraphics[width=\linewidth]{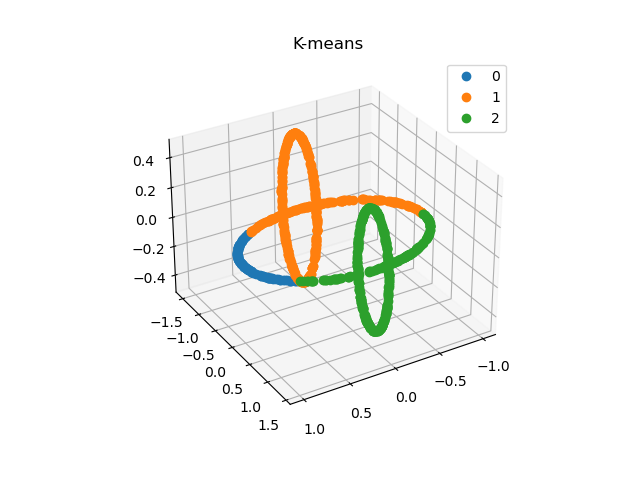}
\end{subfigure}
\caption{Comparison. No added noise. K-means: 244 mistakes}\label{fig:comparison_0}
\end{figure}

\begin{figure}[h!]
\centering
\begin{subfigure}[]{0.25\linewidth}
\includegraphics[width=\linewidth]{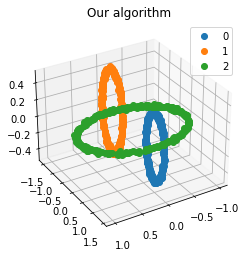}
\end{subfigure}
\begin{subfigure}[]{0.39\linewidth}
\includegraphics[width=\linewidth]{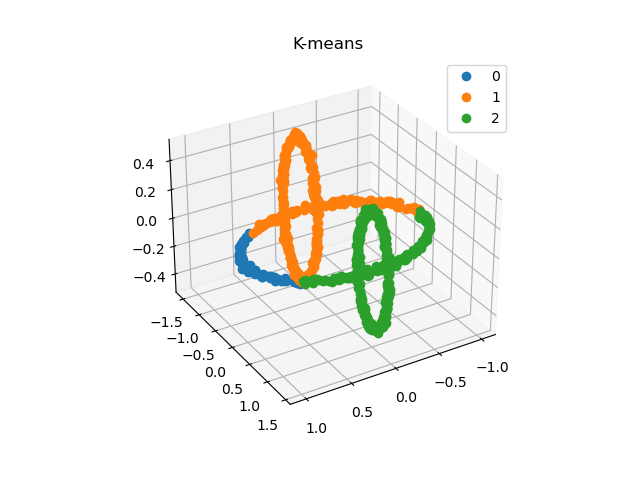}
\end{subfigure}
\caption{Comparison, SD=0.01. K-means: 233 mistakes}\label{fig:comparison_1}
\end{figure}

\begin{figure}[h!]
\centering
\begin{subfigure}[]{0.25\linewidth}
\includegraphics[width=\linewidth]{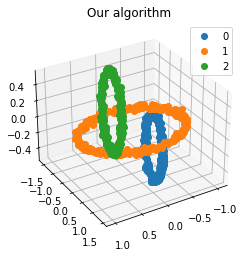}
\end{subfigure}
\begin{subfigure}[]{0.39\linewidth}
\includegraphics[width=\linewidth]{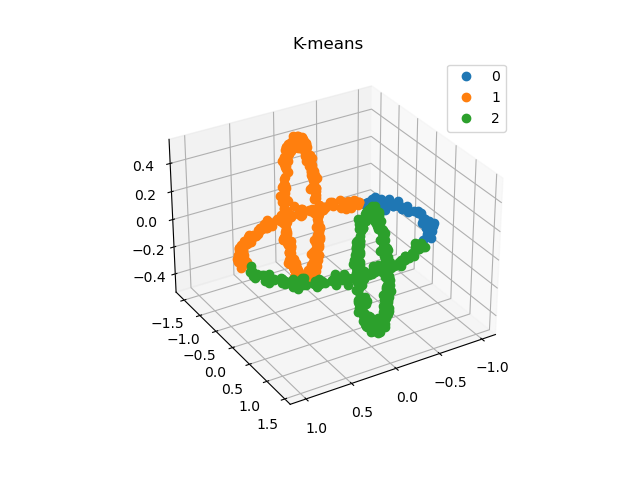}
\end{subfigure}
\caption{Comparison, SD=0.02. K-means: 255 mistakes}
\label{fig:comparison_2}
\end{figure}

\begin{figure}[h!]
\centering
\begin{subfigure}[]{0.25\linewidth}
\includegraphics[width=\linewidth]{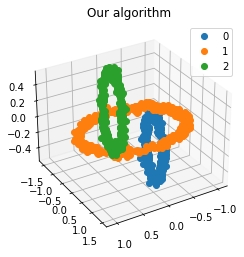}
\end{subfigure}
\begin{subfigure}[]{0.39\linewidth}
\includegraphics[width=\linewidth]{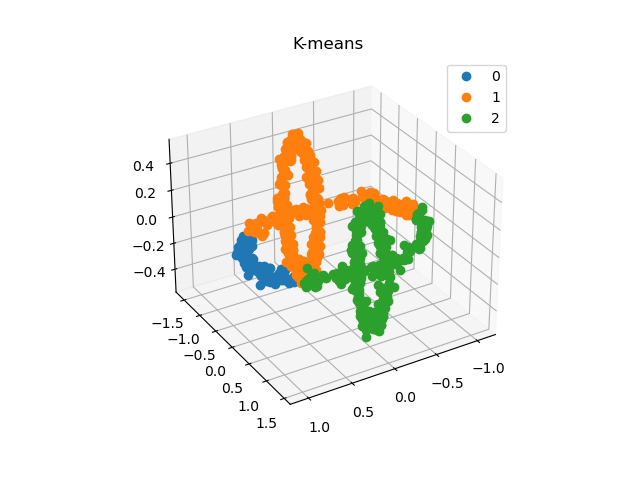}
\end{subfigure}
\caption{Comparison, SD=0.03. K-means: 237 mistakes}\label{fig:comparison_3}
\end{figure}

\begin{figure}[h!]
\centering
\begin{subfigure}[]{0.25\linewidth}
\includegraphics[width=\linewidth]{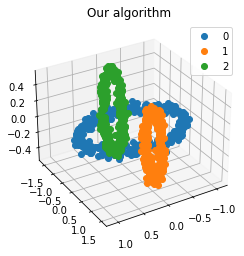}
\end{subfigure}
\begin{subfigure}[]{0.39\linewidth}
\includegraphics[width=\linewidth]{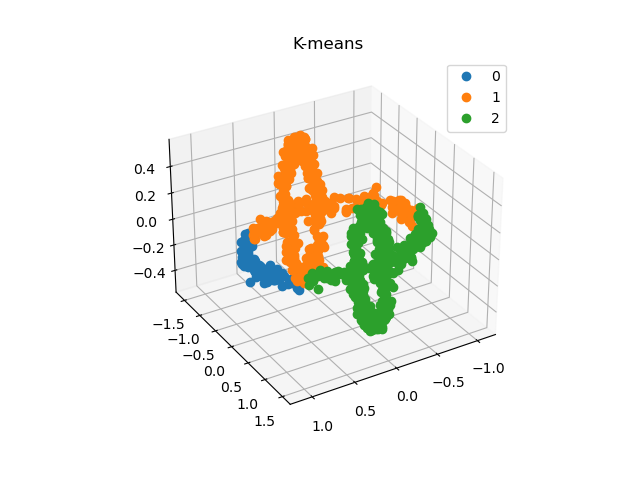}
\end{subfigure}
\caption{Comparison, SD=0.04. K-means: 247 mistakes}\label{fig:comparison_4}
\end{figure}

\begin{figure}[h!]
\centering
\begin{subfigure}[]{0.25\linewidth}
\includegraphics[width=\linewidth]{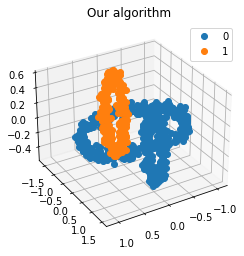}
\end{subfigure}
\begin{subfigure}[]{0.39\linewidth}
\includegraphics[width=\linewidth]{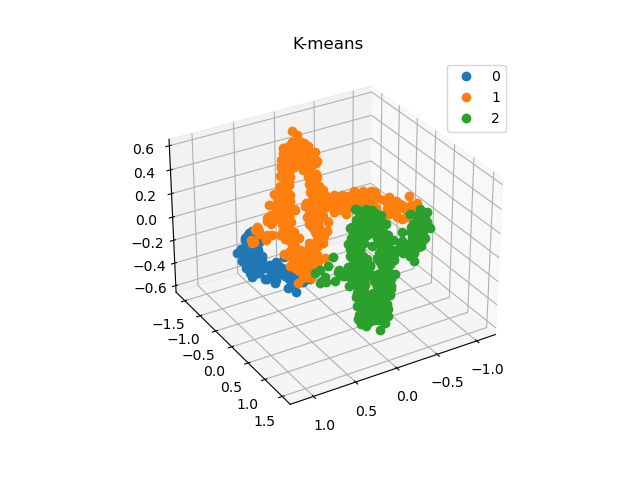}
\end{subfigure}
\caption{Comparison, SD=0.05. K-means: 248 mistakes}
\label{fig:comparison_5}
\end{figure}

\subsection{Test on Image Data}
In order to test our algorithm on image data, we used the unprocessed images from Columbia University Image Library (COIL-20) database. These images form a collection of 448×416-pixel gray scale images from 5 objects, each of which is photographed
at 72 different rotation angles \cite{COIL-20}. 

Regarding each image as a vector of pixel intensities in $\mathbb{R}^{448\times 416}$ yields a set $X \subset \R^{448\times 416}$ with 360 points; this set becomes a finite metric space when endowed with
the ambient Euclidean distance. The result of applying our clustering algorithm to this set is the correct classification of all images in five clusters. We also calculated the results of applying the $k$-means and the $k$-NN algorithms to the same set. In the first case, we gave the data $k=5$ as input and got an incorrect classification of $13$ images. In the second case, we also took $k=5$, obtaining an incorrect classification of $47$ images. The results of each of these experiments are summarized in Figures \ref{COIL}-\ref{kmeans_coil}.

\begin{figure}[h!]
\centering
\begin{subfigure}[]{0.4\linewidth}
\includegraphics[width=\linewidth]{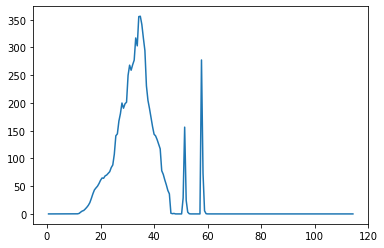}
\end{subfigure}
\begin{subfigure}[]{0.4\linewidth}
\includegraphics[width=\linewidth]{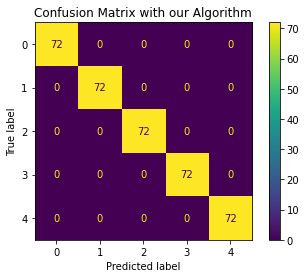}
\end{subfigure}
\caption{Performance of our algorithm for the COIL-20 unprocessed image data. Left: the Entropy graph, right: the Confusion Matrix. }
\label{COIL}
\end{figure}

\begin{figure}[h]
    \centering
    \begin{subfigure}[]{0.4\linewidth}
    \includegraphics[width=\linewidth]{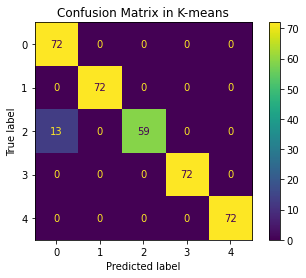}
    \end{subfigure}
    \caption{Confusion Matrix for the $k$-means algorithm applied to the COIL-20 unprocessed image data.}
    \label{kmeans_coil}
\end{figure}

\subsection{Dimension Reduction}

In order to test our dimension reduction algorithm, we tried the algorithm 
on several figures in three dimensions and reduced them to two dimensions. The results are found in Figures \ref{fig:corona}-\ref{fig:swiss_red}, where we
see that the local geometry of the circular figures was largely preserved, and for the two-dimensional surfaces, points which were close in three-dimensions mostly stayed close in two-dimensions.
\begin{figure}[h!]
\centering
\begin{subfigure}[]{0.45\linewidth}
\includegraphics[width=\linewidth]{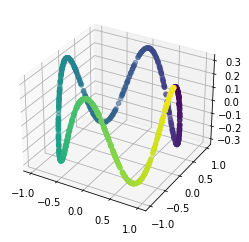}
\end{subfigure}
\begin{subfigure}[]{0.45\linewidth}
\includegraphics[width=\linewidth]{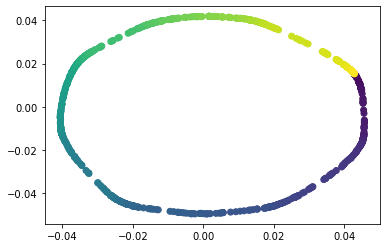}
\end{subfigure}
\caption{The corona and its two-dimensional reduction.}\label{fig:corona}
\end{figure}

\begin{figure}[h!]
\centering
\begin{subfigure}[]{0.45\linewidth}
\includegraphics[width=\linewidth]{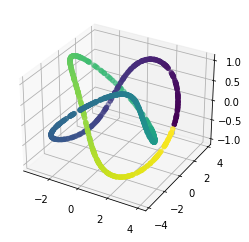}
\end{subfigure}
\begin{subfigure}[]{0.45\linewidth}
\includegraphics[width=\linewidth]{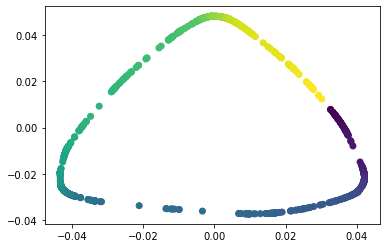}
\end{subfigure}
\caption{A trefoil knot and its two-dimensional reduction}\label{fig:trefoil}
\end{figure}

\begin{figure}[h!]
\centering
\begin{subfigure}[]{0.45\linewidth}
\includegraphics[width=\linewidth]{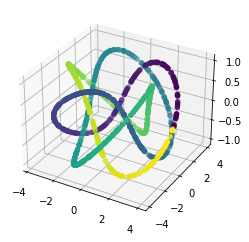}
\end{subfigure}
\begin{subfigure}[]{0.45\linewidth}
\includegraphics[width=\linewidth]{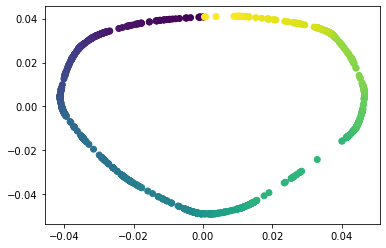}
\end{subfigure}
\caption{A torus knot and its two-dimensional reduction}\label{fig:torusKnot}
\end{figure}


\begin{figure}[h!]
\centering
\begin{subfigure}[]{0.45\linewidth}
\includegraphics[width=\linewidth]{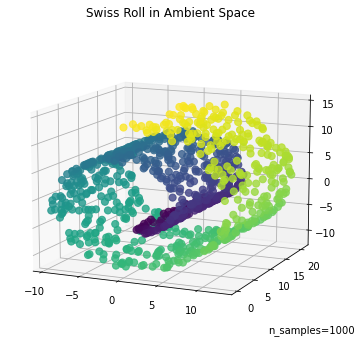}
\end{subfigure}
\begin{subfigure}[]{0.45\linewidth}
\includegraphics[width=\linewidth]{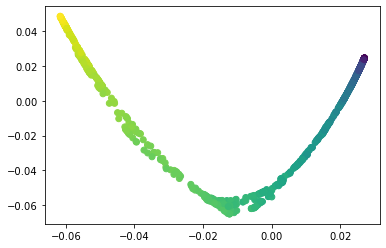}
\end{subfigure}
\caption{The Swiss roll.}\label{fig:swiss_red}
\end{figure}

\section{Discussion and future work}

In this article, we have presented new clustering and dimension reduction algorithms for a data set $S$ sampled from a
uniform distribution on a metric measure space $X$, possibly corrupted by Gaussian noise, where $X$ has been embedded in a larger metric space $Z$, and such that the metrics on $X$ and $Z$ are similar at small enough scales. The algorithms work by constructing a family of graphs indexed by the positive real numbers $r > 0$ which, roughly speaking, indicate the size of a local neighborhood around each point in the sample. We then identify an
optimal scale $\hat{r} > 0$ by maximizing the relative von Neumann entropy of
specially constructed heat operators based on the graphs. For clustering, we identify
the clusters as the components of the associated graph best approximate
the same connected components as the space $X$, and for dimension reduction, we use the eigenvectors of the Laplacian $L_{\hat{r}}$ to construct a map
$\Psi:S \to \R^k$. We have shown that the clustering algorithm represents a significant improvement over the Average Relative Entropy Method of \cite{Rieser_FODS_2021}, in addition to outperforming $k$-means clustering on the examples which we have shown here. A particularly interesting aspect of our construction is that the weights in our graphs are simply chosen to be the pairwise distance in the ambient space, in contrast to other spectral methods such as Laplacian Eigenmaps \cite{Belkin_Niyogi_2003} and Diffusion Maps \cite{Coifman_Lafon_2006}, where an ambient heat kernel is used.

There are a number of benefits to considering relative
von Neumann entropy instead of an average of classical relative entropy as in
\cite{Rieser_FODS_2021}, or even the semigroup-based heuristic of \cite{Shan_Daubechies_2022-arXiv}. In particular, von Neumann entropy is a natural noncommutative construction on the (normalized) heat operators, and we believe that this will make its rigorous theoretical treatment more tractable than the methods in either of \cite{Rieser_FODS_2021} or \cite{Shan_Daubechies_2022-arXiv}, in addition to providing motivation and a setting for studying more noncommutative techniques in statistics.

We note three issues with this method which we hope to address in future work. First, the success of the clustering algorithm presented here depends strongly on
the assumption that the sampling distribution is well-spread-out on its support - in this case, we used a uniform distribution - which unfortunately 
is not fulfilled in many interesting real-world examples. Extending this
method to non-uniform distributions, in addition to dealing with a wider
range of noise models, is an important avenue for future research. We also
note that the method proposed here is, from a certain point of view, a refinement of
single-linkage clustering (see \cite{Everitt_etal_2011}, Section 4.2), and as such, it shares many of its
shortcomings, in particular that it will produce `chaining' artifacts that may occur in single-linkage clustering. Nonetheless, we also expect that the solutions which
have been found for these issues in the case of single-linkage clustering will also work here with the appropriate modifications.

A further issue with the method which is currently unresolved is that 
computing the eigenvalues of large, dense matrices is computationally 
expensive. However, since the maxima of the relative von Neumann entropy 
appear to occur for
relatively small values of the scale parameter $r>0$, we are optimistic that
future investigation will eliminate the need to consider matrices
which are not sparse. Indeed, one may simply impose sparseness of the graphs
as a constraint of the algorithm, and then explore the effectiveness of this
modification, but we believe that theoretical results that justify restricting 
consideration to sparse graphs may also be possible to find. Given the results 
we observed in our current experiments, however, we do
not expect that restricting the algorithm to only consider sparse graphs would
significantly affect the accuracy of the algorithms, but we do expect that it
would significantly improve their speed.

Finally, the empirical success of the techniques introduced here introduces 
many interesting, difficult questions of how to establish performance 
guarantees for these methods.
\section*{Conflict of Interest Statement}

The authors certify that they have no conflict of interest in the subject 
matter or materials discussed in this article.

\printbibliography

\end{document}